\documentclass{article}
\usepackage{ijcai20}
\usepackage{macros_pan}

\newcommand{\cupdot}{\mathbin{\mathaccent\cdot\cup}}

\newcommand{\KK}{\mathcal{K}}
\newcommand{\cK}{\mathcal{K}}
\newcommand{\alg}{\ensuremath{\mathsf{SAMP}}\xspace}

\newcommand{\rd}{\ensuremath{\operatorname{RD}}\xspace}

\newcommand{\ts}{\ensuremath{\operatorname{TS-RA}}\xspace}
\let\oldnl\nl% Store \nl in \oldnl
\newcommand{\nonl}{\renewcommand{\nl}{\let\nl\oldnl}}%

\usepackage{times}  %Required
\usepackage{soul}
\usepackage{url}  %Required
\usepackage[hidelinks]{hyperref}
\hypersetup{colorlinks, citecolor=red, filecolor=blue, linkcolor=blue, urlcolor=blue}
\usepackage{cleveref}
\usepackage[utf8]{inputenc}
\usepackage[small]{caption}
\usepackage{graphicx}  %Required
\usepackage{subfigure}
\usepackage{amsmath}
\usepackage{booktabs}
\title{A Unified Model for the Two-stage Offline-then-Online Resource Allocation\thanks{Copyright \copyright~2020 International Joint Conferences on Artificial
Intelligence (IJCAI). All rights reserved.}}
\author{
Yifan Xu$^{1,2,4}$
\and
Pan Xu$^3$\and
Jianping Pan$^4$\And
Jun Tao$^{1,2}$
\affiliations
$^1$Key Lab of CNII, MOE, Southeast University, Nanjing, China\\
$^2$School of Cyber Science and Engineering, Southeast University, Nanjing, China\\
$^3$Department of Computer Science, New Jersey Institute of Technology, Newark, USA\\
$^4$Department of Computer Science, University of Victoria, Victoria, Canada\\
\emails
xyf@seu.edu.cn,
pxu@njit.edu,
pan@uvic.ca,
juntao@seu.edu.cn
}

\begin{document}

\maketitle

\begin{abstract}
With the popularity of the Internet, traditional offline resource allocation has evolved into a new form, called \emph{online resource allocation}. It features the online arrivals of agents in the system and the real-time decision-making requirement upon the arrival of each online agent. Both offline and online resource allocation have wide applications in various real-world matching markets ranging from ridesharing to crowdsourcing. There are some emerging applications such as \emph{rebalancing in bike sharing} and \emph{trip-vehicle dispatching in ridesharing}, which involve a two-stage resource allocation process. The process consists of an offline phase and another sequential online phase, and both phases compete for the same set of resources.
In this paper, we propose a unified model which incorporates both offline and online resource allocation into a single framework. 
Our model assumes \emph{non-uniform and known} arrival distributions for online agents in the second online phase, which can be learned from historical data.
We propose a parameterized linear programming (LP)-based algorithm, which is shown to be at most a constant factor of $1/4$ from the optimal.
Experimental results on the real dataset show that our LP-based approaches outperform the LP-agnostic heuristics in terms of robustness and effectiveness.
\end{abstract}

\section{Introduction}\label{sec:intro}

Matching markets involve heterogeneous agents (typically from two parties) who are paired for mutual benefits.  During the last decade, matching markets have  emerged and grown rapidly through the medium of the Internet. They have evolved into a new style, 
called \emph{Online Matching Markets} (OMMs). Typical examples include ridesharing (riders and drivers), crowdsourcing (workers and tasks), and  Internet advertising (impressions and advertisers). In OMMs, agents of at least one type join the market in an online fashion, and are referred to as online agents (\eg riders, workers, and impressions). 
Furthermore, upon the arrival of any online agent, we have to decide quickly and irrevocably 
which offline agent(s) to match it with. That is mainly due to the low ``patience'' of the online agents. These features---online arrivals and the real-time decision-making requirement---distinguish OMMs from traditional matching markets where the information of all agents is fully disclosed in advance.

OMMs have received significant interest in both computer science and operations research communities. There is a large body of research work who studied matching policy design for the profit maximization in ridesharing~\cite{ashlagi2019edge,Patrick-18-JAI,BeiZ18,aaai-19-stable,DickersonAAAI18,Li2020TripVehicleAA}, crowdsourcing~\cite{assadi2015online,ho2012online,AAMAS18}, admission scheduling  and online recommendations~\cite{wang2018on-adv,ma2017online,chen2016dynamic}. Recently, \cite{dickerson2019online} presented a general  online resource allocation model, called Multi-Budgeted Online Assignment (MBOA), to address the matching policy design in various real-world OMMs featuring that each assignment could potentially consumes multiple resources. The basic model is as follows.  We are given a bipartite graph $G=(I,J, E)$ where $I$ and $J$ represent the respective sets of offline and online agents. There is a set $\cK$ of $K$ resources and each resource $k \in \cK$ has a given budget $B_k$. Each edge $e=(i,j)$ or assignment\footnote{Throughout this paper, we use the two terms ``edge'' and ``assignment'' interchangeably.} of $j$ to $i$ is associated with a profit and a vector cost of dimension $K$. A natural question is how to arrange an assignment for each online agent upon its arrival such that the total expected profit is maximized subject to the budget constraints. 

MBOA~\cite{dickerson2019online} has greatly generalized the assignment problems in OMMs. Consider Amazon Mechanical Turk for example, which automatically crowdsources online workers for offline tasks. In this case, there are typically two kinds of resources, budget (deposited by a task manager to pay workers) and tasks (typically each task has a limited number of copies). In addition to OMMs, MBOA  has captured a wide range of online resource allocation problems in datacenters~\cite{ghodsi2012multi,joe2013multiresource,ghodsi2013choosy}, 
public safety~\cite{shumate1966quantitative,Lee1979,recentPolice} and candidates recruitment~\cite{chenthamarakshan2012systems,yi2007matching}, for example.  Note that MBOA considers resource allocation only for online arrival agents.  
In some emerging applications, however, there are \emph{both  offline and online agents} who are competing for the same set of resources. Consider the two motivating examples below.

\xhdr{Rebalancing in Bike Sharing Systems (BSSs)}. BSSs have become an
important alternative for addressing the last mile problem in city Intelligent
Transportation Systems. Users varied travel patterns lead to uneven distributions of bikes among docking stations. A common scenario is during afternoon peak hours, ``demanding'' stations (\eg near shopping malls or metro stations) have a high demand but a low supply, while ``supplying'' stations (\eg suburb areas far away from downtown) have a high supply but a low demand. Thus, BSSs have to rebalance distributions by moving bikes from  ``supplying'' to ``demanding'' stations such that they can address as many user requests as possible. Two common rebalancing strategies are proposed. One is the Truck-Based Rebalancing (TBR), which directly utilizes trucks or trailers to move bikes from one station to the other. The other is Crowdsourcing-Based Rebalancing (CBR), which incentives online bike users to participate in rebalancing
tasks. TBR has a strong controllability, high efficiency during off-peak hours, and a high labor cost. In contrast, CBR has a relatively low controllability, high efficiency during peak hours, and a low labor cost. The two approaches each has received considerable interest, see, \eg~\cite{raviv2013static,o2015data,liu2016rebalancing,li2018dynamic} for TBR, 
and~\cite{singla2015incentivizing,haider2018inventory,pan2019drl,duan2019optimizing} for CBR. However, very few of them have ever considered both TBR and CBR in the same framework,
though the two are proposed to address the same issue. For TBR, it is typically deployed during off-peak hours and involves renting
\emph{offline} trucks and labors, especially when there are not many users. For CBR, it is operated during peak hours instead, and crowdsources \emph{online} users as potential workers, especially when trucks are restricted during peak hours. Note that both TBR and CBR compete for one resource, the global budget used for renting truck and hiring labors in TBR, and paying online workers in CBR. Additionally, the two compete for the same set of tasks, \ie moving bikes from supplying to demanding stations. In fact, we can view each supplying and demanding station as a ``resource''  with the capacity being the number of bikes in supply and demand. A natural question arises: How to optimally allocate the ``resources'' to the two approaches such that we can complete as many tasks as possible?

\xhdr{Trip-vehicle dispatching with multi-type requests}~\cite{huang2019optimal}. In recent years, ridesharing apps such as DiDi launch services that allow riders to reserve a trip in advance, which is called \emph{scheduled requests}. 
Consider a short time window $\cW$ during peak hours. Let $\cC_1$ be the set of scheduled requests received with starting time in $\cW$ and $\cC_2$ be the (random) set of real-time  requests which arrive in $\cW$. Note that $\cC_1$ is received well before the start of $\cW$ and thus can be viewed as \emph{offline} requests. Both $\cC_1$ and $\cC_2$ are competing for the same set of ``resources'': the available drivers in $\cW$. A natural question is: How to allocate drivers to requests 
such that the total expected profit obtained is maximized?

The above two examples both feature a two-stage resource allocation: the first stage involves a set of \emph{static} offline agents while the second stage involves a set of \emph{dynamic} online agents, and they are competing for the same set of resources.

\xhdr{Main contributions}. Our contributions are summarized as follows. First, we propose a unified framework, called \emph{Two-Stage Resource Allocation} (\ts), by generalizing the model in~\cite{dickerson2019online}. Overall, \ts incorporates the offline and online resource allocation into one single framework, which can potentially capture a wider range of applications and optimize the allocation jointly. Second, we propose a parameterized linear programming (LP)-based algorithm, which is shown to be at most a constant factor of $1/4$ from the optimal. One highlight is the performance of our LP-based approaches depends only on the \emph{sparsity}, \ie the maximum number of different resources requested by an assignment, regardless of the total number of resources involved which can be potentially large.
Third, we present several greedy-based heuristics and compare our LP-based algorithms against those LP-agnostic heuristics. Experimental results on the real dataset show our LP-based approaches are robust and effective in a wide range of settings and they can universally dominate all LP-agnostic heuristics.

\section{Preliminaries}
\label{sec:pre}

We first formally define the model considered in this paper and then describe the required background for the technical sections of this paper. As a notation, denote $[k] \doteq \{1,2,\ldots,k\}$ for any positive integer $k$.

\xhdr{Two-Stage Resource Allocation} (\ts). 
% We have a bipartite graph $G=(I,H \cupdot J)$ (the RHS of vertices is a disjoint union of $H$ and $J$) \bluee{where $I$ and $H$ represent the sets of offline vertices,
% and $J$ represents the set of \emph{types} of online vertices.
We have a bipartite graph $G$ of vertex sets $(I,H \cupdot J)$ where $I$ and $H$ represent the sets of \emph{types} of offline vertices, 
and $J$ disjointed from $H$ represents the set of \emph{types} of online vertices.
Let $E^1$ and $E^2$ be the respective sets of edges in the subgraphs $G^1$ of $(I, H)$ and $G^2$ of $(I, J)$. There are $K$ offline resources and each resource $k$ has a budget $B_k \in \mathbb{R}^{+}$.  Each edge or assignment $e$ yields a positive reward (weight) $w_e>0$ and incurs a vector-valued cost $\ba_e=(a_{e,k}) \in  [0,1]^K$, where $a_{e,k}$ denotes the consumption of resource $k$ by $e$. The assignment process consists of two \emph{sequential} stages, (offline) Phase I and (online) Phase II. We first select an arbitrary set of edges $E_1$ (can be a multiset of $E^1$) in Phase I and then in Phase II, we conduct an online assignment process on the subgraph $G^2$ with $I$ and $J$ be the respective sets of offline and online vertices and output a set of edges $E_2$ (can be a multiset of $E^2$). Our overall goal is to design an allocation policy to maximize the expected total weights of all edges in $E_1 \cup E_2$ while the total cost of all edges in $E_1 \cup E_2$ does not exceed the budget $\bfB=(B_k)$ (the budget constraint is enforced throughout Phase I and Phase II).

Note that the full graph $G$ of $(I,H \cupdot J)$ is known as part of the input\footnote{A future direction is to explore how to incorporate online learning into our framework to learn the input dynamically.}. The only stochastic part to the algorithm is the arrival sequence of online vertices in Phase II, which is specified as follows. We have a finite horizon $T$ (known in advance). For each time (or round) $t \in [T] \doteq\{1,2,\ldots, T\}$, a vertex of type\footnote{The two terms ``a vertex of type $j$'' and ``a vertex $j$'' will be used interchangeably when the context is clear.}$j \in J$ will be sampled (or $j$ arrives) from a known distribution $\{p_{jt}\}$ such that $\sum_{j \in J} p_{jt}=1$, and the sampling procedures are independent across the $T$ rounds. Upon the arrival of an online vertex $j$, an \emph{immediate and irrevocable} decision is required: either reject $j$, or select an edge $e \in E^2_j$, where $E^2_j \subseteq E^2$ is the set of incident edges to $j$ in $G^2$. Observe that there could be multiple arrivals for each given type of online vertices, we can possibly select multiple copies of a single edge $e \in E^2$ in Phase II.

\noindent \textbf{Remarks}. First, the algorithms presented in this paper are applicable directly to a more general setting where each assignment $e$ yields a \emph{random} reward $W_e$ (independent from others). What we need is just the expectation $w_e \doteq \E[W_e]$ for each $e$. Second, we cannot simply merge Phase I to part of Phase II. Recall that in Phase II, we have to make an irrevocable matching decision upon the arrival of each online vertex. Suppose we try to treat Phase I as part of Phase II by viewing each arrival $h \in H$ with probability $1$ during the first $|H|$ rounds. This reduction will essentially restrict the power of algorithm design in Phase I in the way that we should process all assignments in $E^1$ following the arrival order of vertices of $H$. In fact, we can process all assignments in $E^1$ simultaneously since the graph $G^1$ is fully known in advance.

%Here is the key difference beween the two phases. At the beginning of Phase I, we are full aware the set of vertices arrived, which is $H$. We can select an arbirary multiset $E_1$ of $E^1$ in Phase I as long as budget constaints are satified. In contrast, at the beginning of Phase II we are only aware the set of \emph{types} of vertices (\ie $J$) which will potentiall arrive: the exact arrival sequence is revealed sequentially following a set of known distributions as described before.  

\xhdr{Integral and non-integral resources.} 
For an integral resource $k$, we have that for any edge $e \in E^1 \cup E^2$, $a_{e,k} \in \{0,1\}$ while for a non-integral resource $k$, we have that for any $e \in E$, $a_{e,k} \in [0,1]$. For any integral resource $k$, WLOG we assume that $B_k\in \mathbb{Z}_{+}$. Let $\KK_1=\{1,2,\cdots,K_1\}$ and $\KK_2=\{K_1+1,\cdots, K_1+K_2\}$ denote the set of integral and non-integral resources, respectively. For any edge $e$, let $\cS_e$ be the set of resources requested by $e$, \ie $\cS_e=\{k \in \cK: a_{e,k}>0\}$. In this paper, we assume that $|\cS_e \cap \KK_1| \le \ell_1$ and $|\cS_e \cap \KK_2| \le \ell_2$, where $\ell_1$  and  $\ell_2$ are the integral and non-integral sparsities (\ie for types of resources), respectively. 
All notations are summarized in Table~\ref{table:notations}.

\begin{table}[h]
\footnotesize
\caption{Summary of notations}
\label{table:notations}
\begin{tabular}{ |l|l| }
  \hline
  $T$ & time horizon \\
 % $m$ & Number of servers \\
 % $n$ & Number of job-types \\
  $E^1$ & set of edges in $G^1=(I,H)$ \\
  $E^2$ & set of edges in $G^2=(I,J)$ \\
  $E^2_j$ & set of edges incident to $j \in J$ in $G^2$ \\
  $B_k$ & budget for resource $k$ \\
   $B$ & minimum budget over all \emph{non-integral} resources\\
   $\cK_1$ ($\cK_2$) & set of integral (non-integral) resources \\
  $K_1$ ($K_2$) & total number of integral (non-integral) resources\\
   $K$ &  total number of resources ($K= K_1 + K_2$)\\
     $\cS_e$ & set of resources consumed by assignment $e$\\
  $\ell_1$ ($\ell_2$) & sparsity of integral  (non-integral) resources per edge\\
  $\ba_e$ & $K$-dimensional cost vector consumed by assignment $e$\\
  $a_{e,k}$ & amount of resource $k$ consumed by assignment $e$ \\
     $p_{jt}$ &  probability that a vertex of type $j$ arrives at $t$ in Phase II\\
  \hline
\end{tabular}
\end{table}
% \vspace{-3mm}

Let us briefly show how \emph{rebalancing in BSSs} can be cast under \ts (similarly for trip-vehicle dispatching). Suppose we have two sets, $\cC_1$ and $\cC_2$, which denote the respective sets of supplying and demanding bike stations. Let each $a \in  \cC_1$ have a supply of $c_a\in\mathbb{Z}^{+}$ and each $b\in \cC_2$ have a demand of $c_b \in \mathbb{Z}^{+}$. Define $I=\cC_1 \times \cC_2$, $H=\{h\}$ includes one dummy node, and $J$ is the set of all online worker types. Consider a given $i=(a,b)$ with $a \in \cC_1, b\in \cC_2$. Let the edge $e=(i,h)$ denote the assignment of the task of moving one bike from $a$ to $b$ via TBR and $e=(i,j)$ denote the assignment of the same task to a worker of type $j$ via CBR. In this way,  we have $|\cC_1|+|\cC_2|$ kinds of integral resources and each $a \in \cC_1$ and $b \in \cC_2$ have a budget of $c_a$ and $c_b$, respectively. We have one non-integral resource, which is the global budget used to pay rented trucks, labor fees, and online workers. Though we potentially have a huge number of resources ($|\cC_1|+|\cC_2|+1$), each assignment requests at most two integral resources (the supplying and demanding stations) and one non-integral resource (the global budget), \ie $\ell_1=2, \ell_2=1$. Later we will show that the performance of our allocation policy depends only on the total number of sparsity $\ell_1+\ell_2$, regardless of the total number of resources.

\xhdr{Extension of competitive ratio}. 
Competitive ratio (CR) is a commonly-used metric to evaluate the performance of online algorithms. In our case, we have two stages, offline Phase I and online Phase II. We can extend CR to our setting as follows. Recall that our overall goal is to maximize the total weights of all assignments made in the two phases.
Let $\ALG$ denote an assignment policy for Phases I and II. Consider an input $\cI$ of the problem. Let $\E[\ALG(\mathcal{I})]$ denote the expected profit obtained by $\ALG$ for this instance, where the expectation is over the randomness in the arrival sequence in Phase II and any internal randomness wired in the algorithm. Similarly, let $\E[\OPT(\mathcal{I})]$ denote the expected value of the optimal offline solution (\ie the expected value of optimal assignments for Phases I and II after observing the entire arrival sequence in Phase II). The competitive ratio is defined as $\inf_{\mathcal{I}} \E[\ALG(\mathcal{I})]/\E[\OPT(\mathcal{I})]$.
	
For any maximization problem such as the one studied here, we say \ALG achieves a ratio at least $\alpha \in (0,1)$ if for any input of the problem the expected profit obtained by \ALG is at least a fraction $\alp$ of the offline optimal solution. Typically computing the value of $\E[\OPT(\mathcal{I})]$ directly is hard. A common method to bypass this is to construct a linear program (called \emph{benchmark \LP}) whose optimal value is an upper bound on $\E[\OPT(\mathcal{I})]$ due to the integrality gap. Hence comparing $\E[\ALG(\mathcal{I})]$ to the optimal value of this \LP~gives a lower bound on the competitive ratio. We will now describe the benchmark \LP~used in this paper.

\xhdr{Benchmark LP}. Recall that $E^2_j \subseteq E^2$ is the set of incident edges to $j$ in $G^2$. Consider an offline optimal algorithm. For each $e \in E^1$, let $x_e$ be the expected number of copies that $e$ is selected during Phase I and for each $e \in E^2$ and $t \in [T]$, let $y_{e,t}$ be the probability that edge $e$ is selected at $t$. Consider the following benchmark \LP. 

% \vspace{-4mm}
\begin{alignat}{2}
\label{LP:offline-a}
\text{max}    &  \textstyle \sum_{e \in E^1 } w_{e} x_e+ \sum_{e \in E^2, t\in [T] } w_{e} y_{e,t}   \\
\text{s.t.}  &  \textstyle \sum_{e \in E^2_{j}} y_{e,t} \le  p_{jt},~~    \textstyle \quad\forall j \in J, t \in [T] \label{cons:arr_1}\\
                   & \textstyle  \sum_{e \in E^1 } x_e a_{e,k}+  \sum_{e \in E^2, t\in [T] } y_{e,t} a_{e,k} \le B_k,      \textstyle \quad\forall k \in [K] \label{cons:res_1}\\
                   &\textstyle 0 \le x_{e}, 0 \le y_{e,t} \le 1,~~                   \textstyle \quad\forall e \in E, t \in [T] \label{LP:offline-d}
\end{alignat}
% \vspace{-4mm}
% \vspace{-4mm}

This \LP~can be interpreted as follows. Constraint~\eqref{cons:arr_1} : for any given vertex of type $j$ and time $t$ during Phase II, the probability that we select an edge incident to $j$ is at most the probability that $j$ arrives at time $t$. Constraint~\eqref{cons:res_1} : for any %(integral or non-integral) 
resource $k$,  the expected consumption cannot be larger than its budget ($B_k$). The last constraint~\eqref{LP:offline-d} is due to the fact that all $\{y_{e,t}\}$ are probability values and hence should lie in the interval $[0, 1]$. Also $x_e$ should be non-negative since it denotes the number of copies that $e$ is selected. The above analysis suggests that any offline optimal algorithm, $\{x_e, y_{e,t}\}$ should be feasible for the above LP. Formally, we have Lemma~\ref{lem:bench-lp} which claims that the optimal solution of this \LP~ is an upper bound on the expected offline optimal value. 
\begin{lemma}\label{lem:bench-lp}
The optimal value to \LP~\eqref{LP:offline-a} is a valid upper bound for the offline optimal solution. 
\end{lemma}

\begin{proof}
Consider an offline optimal algorithm. For each $e \in E^1$, let $x_e$ be the expected number of copies that $e$ is selected during Phase I and for each $e \in E^2$ and $t \in [T]$, and let $y_{e,t}$ be the probability that edge $e$ is selected at $t$. Essentially, we need to show that the solution of $\{x_e, y_{e,t}\}$ is feasible to LP~\eqref{LP:offline-a}.

Consider Constraint~\eqref{cons:arr_1} first. For any given vertex of type $j$ and time $t$ during Phase II, the probability that we select an edge incident to $j$ should be at most the probability that $j$ arrives at time $t$. Thus Constraint~\eqref{cons:arr_1} is valid. Now consider Constraint~\eqref{cons:res_1}. For any (integral or non-integral) resource $k$,  the expected consumption cannot be larger than its budget ($B_k$). As for Constraint~\eqref{LP:offline-d}, that is due to the fact that all $\{y_{e,t}\}$ are probability values and hence should lie in the interval $[0, 1]$. Also $x_e$ should be non-negative since it denotes the number of copies that $e$ is selected.  The above analysis suggests that any offline optimal solution $\{x_e, y_{e,t}\}$ should be feasible for \LP~\eqref{LP:offline-a}. Thus, we claim that the optimal solution of \LP~\eqref{LP:offline-a} is an upper bound on the expected offline optimal value.
\end{proof}

\section{Main Algorithm and Results}
In this section, we describe our main  algorithm $\alg$, which is an LP-guided sampling policy. 
		
Let $\{x_e^*, y^*_{e,t} \}$ be an optimal solution to the LP~\eqref{LP:offline-a}. The main idea behind $\alg$ (described in Algorithm~\ref{alg:main}) is as follows.
For a given non-negative value $x \ge 0$, let $\rd(x)$ be such an integral random variable that $\rd(x)=\lceil x_e \rceil$ with probability $x-\lfloor x \rfloor$ and $\rd(x)=\lfloor x \rfloor$ otherwise. Essentially $\rd(x)$ takes values between $\lfloor x \rfloor$ and $\lceil x \rceil$ with $\E[\rd(x)]=x$. 
In the offline Phase I, we independently sample $\rd(\eta x_e^*)$ copies of edges for each $e \in E^1$, where $\eta \in (0,1]$ is a scaling factor to optimize later. Let $\cE$ be the random multiset of all edges sampled in Phase I. Notice that the total consumption of $\cE$ can potentially violate budget constraints for some resources. If so, then remove all copies of $e$ if any number of $e$ contributes to the violation\footnote{Theoretically, we can act in this way, though we can implement the removal of copies sequentially until no violation of budget.}. In the online Phase II, suppose a vertex (of type) $j$ arrives at time $t$. We say an edge $e \in E^2_j$ is \emph{safe} iff adding $e$ will not violate any budget constraint (\ie the remaining budget at $t$ is enough to cover all resources requested by $e$). Sample an edge $e$ from $E_j^2$ with probability $\alp y_{e,t}^*/p_{j,t}$ and select it iff $e$ is safe. Here $\alp \in (0,1]$ is scaling factor to optimize later.

\begin{algorithm}[ht]
% \caption{A two-phase LP-based sampling algorithm $\alg(\eta, \alp)$} 
\caption{$\alg(\eta, \alp)$} 
\label{alg:main}
\DontPrintSemicolon
\nonl \textbf{Offline Phase I:}\;
Independently sample $\rd(\eta x_e^*)$ copies of edges for each $e \in E^1$.\;
Remove all copies of $e$ if $e$ is not \emph{safe} (\ie $e$ requests some resource on which there is a budget overrun).\;
\nonl \textbf{Online Phase II:}\;
Assume a vertex (of type) $j$ arrives at time $t$.\;
Sample an edge $e$ from $E_j^2$ with probability $\alp y_{e,t}^*/p_{jt}$. \label{alg:step3}\;
If $e$ is safe (\ie adding $e$ will not break any budget constraint at $t$), then select it; otherwise reject it.\;
\end{algorithm}

% \vspace{-3mm}

Note that Step~\ref{alg:step3} of \alg is well defined since we have $\sum_{e \in E^2_j} \alpha y^*_{e,t}/p_{jt} \le \sum_{e \in E^2_j}  y^*_{e,t}/p_{jt} \le 1$, due to Constraint
~\eqref{cons:arr_1} in \LP~\eqref{LP:offline-a}.

\xhdr{Main theoretical results}.
 \begin{theorem}[Performance of $\alg$ for the integral case]
\label{thm:main-1}
For \ts when all resources are integral with a sparsity of $\ell$, 
$\alg$ with $\eta=\alp=\frac{1}{2\ell}$ achieves a competitive ratio of at least $ \frac{1}{4 \ell}$ using \LP~\eqref{LP:offline-a} as the benchmark.  
\end{theorem}

\begin{theorem}[Performance of $\alg$ for the general case]
\label{thm:main-2}
For \ts when both integral and non-integral resources are involved with  respective sparsities of $\ell_1$ and $\ell_2$, 
$\alg$ with $\eta=\alp=\frac{1}{2\ell}$  achieves a competitive ratio of at least $ \frac{1}{4 \ell} (1-\ep)$ using \LP~\eqref{LP:offline-a} as the benchmark, where $\ell=\ell_1+\ell_2$ and $\ep=2\max \big(1/(B-2), \ell_2 \exp(-B/12+1/6) \big)$ with $B$ being the minimum budget over all non-integral resources.  
\end{theorem}

The main technique used here is a combination of randomized rounding with alterations introduced to solve packing integer programs~\cite{bansal2012solving} and randomized sampling for online resource allocation~\cite{dickerson2019online}. Generally, we cannot beat the ratio of $1/(\ell-1+1/\ell)$ for the integral case with a sparsity of $\ell$ using LP~\eqref{LP:offline-a} as the benchmark, due to the hardness result derived from the set packing problem~\cite{furedi1993}. This suggests that the performance of \alg is almost a constant factor ($1/4$) from the optimal for the integral case. For the general case, Theorem~\ref{thm:main-2} implies that we can simply treat non-integral resources as integral and $\alg$ will achieve nearly the same ratio when the budgets among all non-integral resources are moderately large $\Omega(1/\ep)$. Note that in the absence of a  lower bound on the budgets of all non-integral resources, LP-based randomized samplings can have an arbitrarily bad performance for online resource allocation, as shown in~\cite{dickerson2019online}.

\vspace{-1mm}
% \section{Performance Analysis of $\alg$}\label{sec:int}

\subsection{Integral Resources Only}\label{sec:int}
In this section, we consider the simple case when all resources are integral with a sparsity of $\ell$.
% \bluee{We consider the simple case when all resources are integral with a sparsity of $\ell$.}

For each $e \in E^1$, let $X_e$ be the random number of copies of $e$ selected in Phase I and for each $e \in E^2$ and $t \in [T]$, let $Y_{e,t}$ indicate if $e$ is selected at $t$ in Phase II by $\alg$. Suppose for a given target constant $\gam \in (0,1)$, we can show  that $\E[X_e] \ge \gam x_e^*$ for all $e \in E^1$ and $\E[Y_{e,t}] \ge \gam y_{e,t}^*$ for all $e \in E^2$ and $t\in [T]$. Then by linearity of expectation, the expected total rewards should be at least $\gam \Big(\sum_{e \in E^1} w_e x_e^*+ \sum_{e \in E^2} w_e y_{e,t}^*\Big)$, which is followed by that $\alg$ achieves an online ratio at least $\gam$ from Lemma~\ref{lem:bench-lp}.

For Phase I, we have the following lemma. 
\begin{lemma}\label{lem:ratio-1}
For each $e\in E^1$, we have $\E[X_e] \ge \eta (1-\eta)^\ell x_e^*$.
\end{lemma}
\vspace{-2mm}
\begin{proof}
According to \alg, we first sample $\rd(\eta x_e^*)$ copies of edge $e$ and then include all of them iff $e$ is  \emph{safe}, \ie all resources requested by $e$ have no budget overrun. Consider a given $e \in E^1$ and recall that $\cS_e \subseteq \cK_1$ is the set of all integral resources requested by $e$. By the sparsity assumption, we have $|\cS_e| \le \ell$. 

Consider a given $k \in \cS_e$ and thus $a_{e,k}=1$. For each $e' \in E^1$ including $e'=e$, let $Z_{e'}=\rd(\eta x^*_{e'})$.   
Notice that 
\begin{align}
 &\textstyle \Pr\Big[ \sum_{e' \in E^1} Z_{e'} a_{e',k}>B_k~|~ Z_e=\lceil \eta x_e^* \rceil \Big] \label{lem1:ineq1}\\
 &=
\textstyle \Pr\Big[ \sum_{e' \in E^1, e'\neq e} Z_{e'} a_{e',k}\ge B_k+1-\lceil \eta x_e^* \rceil \Big] \label{lem1:ineq2}\\
 & \le \textstyle \Pr\Big[ \sum_{e' \in E^1, e'\neq e} Z_{e'} a_{e',k}\ge B_k-x_e^*\Big] \label{lem1:ineq3}\\
 & \le \textstyle \E\Big[ \sum_{e' \in E^1, e'\neq e} Z_{e'} a_{e',k} \Big]/(B_k-x_e^*) \le \eta \label{lem1:ineq4}
\end{align}

Inequality~\eqref{lem1:ineq3} is due to $\lceil \eta x_e^* \rceil \le \lceil  x_e^* \rceil \le x_e^*+1$. Inequality~\eqref{lem1:ineq4} is due to Markov's inequality and the fact that $\E\Big[ \sum_{e' \in E^1, e'\neq e} Z_{e'} a_{e',k} \Big]=\sum_{e' \in E^1, e' \neq e} \eta x_{e'}^* a_{e',k} \le \eta (B_k-x_e^*)$ from Constraint~\eqref{cons:res_1} in \LP~\eqref{LP:offline-a}. We can verify that in the other case when $Z_e$ gets rounded down to $\lfloor \eta x_e^* \rfloor$, the probability that resource $k$ has a budget overrun should be no larger than that when $Z_e$ is rounded up to $\lceil \eta x_e^* \rceil $. Therefore, we claim that regardless the outcomes of $Z_e=z_e \in \{ \lceil \eta x_e^* \rceil, \lfloor \eta x_e^* \rfloor \}$, there is a budget overrun on resource $k$ with probability at most $\eta$. 

Observe that $\{Z_{e'}|e'\in E^1\}$ are all independent random variables and events $\{ (\sum_{e' \in E^1, e'\neq e} Z_{e'} a_{e',k}\ge B_k-z_e a_{e,k})| k \in \cS_e\}$ are decreasingly likely events for each given $z_e$. Thus applying the FKG inequality~\cite{fortuin1971}, the probability that $e$ is safe should be 
$$\textstyle \Pr\Big[ \bigwedge_{k \in \cS_e} \Big(\sum_{e' \in E^1, e'\neq e} Z_{e'} a_{e',k}\ge B_k-z_e a_{e,k} \Big)\Big] \ge (1-\eta)^\ell$$

Note that the above inequality is valid regardless the choices of $z_e \in \{ \lceil \eta x_e^* \rceil, \lfloor \eta x_e^* \rfloor \}$. Therefore, $\E[X_e] \ge \E[Z_e ]  \cdot (1-\eta)^\ell  = \eta (1-\eta)^\ell x_e^*$. 
\end{proof}

For Phase II, we have the following lemma. 
\begin{lemma}\label{lem:ratio-2}
For each $e\in E^2$ and $t \in [T]$, we have $\E[Y_{e,t}] \ge \alp \Big(1-\ell \cdot \max(\eta, \alp)\Big) y_{e,t}^*$.
\end{lemma}
\begin{proof}
Consider a given $t \in [T]$, $e=(i,j)\in E^2$ and $k \in \cS_e$. Let $\tau_1=\sum_{e' \in E^1} x_{e'}^* a_{e',k}$ and $\tau_2=\sum_{e' \in E^2, t'<t} y_{e',t'}^* a_{e',k}$. From  Constraint~\eqref{cons:res_1} in \LP~\eqref{LP:offline-a}, $\tau_1+\tau_2 \le B_k$. 

Assume $j$ arrives at time $t$, which occurs with probability $p_{jt}$. Let $U_1$ and $U_{2,t}$ be the respective (random) consumptions of resource $k$ during Phase I and at the beginning of time $t$ during Phase II. Notice that the event occurs that $e$ is safe at $t$ with respect to resource $k$, denoted by $\SF_{e,t,k}$, iff $U_1+U_{2,t} \le B_k-1$ (since $a_{e,k}=1$). 
\begin{align} \label{eqn:lem2}
\textstyle \Pr[\SF_{e,t,k}] &=
\textstyle 1-\Pr[U_1+U_{2,t} \ge B_k] \ge 1-\frac{\E[U_1+U_{2,t}]}{B_k}
\end{align}

Observe that $\E[U_1] \le \sum_{e' \in E^1} \E[\rd(\eta x^*_{e'}) \cdot a_{e',k}] \le \eta \tau_1$. For each $t'<t$ and $j \in J$, let $\bo_{j,t'}$ indicate if $j$ arrives at $t'$. For each $e' \in E^2_j$, let $\bo_{e',t'}$ indicate if $e'$ is sampled when $j$ comes at $t'$. Thus,
\begin{align*}
&\textstyle \E[U_{2,t}]  \le \sum_{t'<t, j \in J, e' \in E^2_j} \E[\bo_{j,t'} \cdot \bo_{e',t'} \cdot a_{e',k}]\\
&\textstyle \le \sum_{t'<t, j \in J, e' \in E^2_j} p_{j,t'} \cdot (\alp y_{e',t'}^*/p_{j,t'}) \cdot a_{e',k} \le \alp \tau_2
\end{align*}
Substituting results on $\E[U_1]$ and $\E[U_{2,t}]$ back to Equation~\eqref{eqn:lem2}, we have 
$$\textstyle \Pr[\SF_{e,t,k}] \ge 1-\frac{\eta \tau_1+\alp \tau_2}{B_k} \ge \textstyle 1-\max(\eta, \alp)$$

The last inequality is due to $\tau_1+\tau_2 \le B_k$. Therefore, for the edge $e=(i,j)$,
\begin{align*}
\textstyle \E[Y_{e,t}]&=\Pr[(\bo_{j,t}=1) \wedge (\bo_{e,t}=1) \wedge (\wedge_{k \in \cS_e}\SF_{e,t,k})] \\
& \ge \textstyle p_{jt} \cdot (\alp y_{e,t}^*/p_{jt})\cdot \Big(1-\ell \cdot \max(\eta, \alp)\Big)
\end{align*}
Thus, we get our claim.
\end{proof}

% \vspace{-4mm}
\xhdr{Proof of main Theorem~\ref{thm:main-1}}. 
\begin{proof}
From Lemmas~\ref{lem:ratio-1} and \ref{lem:ratio-2}, the final online competitive ratio achieved by $\alg(\eta, \alp)$ should be $\min \Big( \eta(1-\eta)^\ell,  \alp\big(1-\ell \cdot \max(\eta, \alp)\big)\Big)$. By choosing $\eta=\alp=1/(2\ell)$, we see that  the final ratio is $1/(4\ell)$. 
\end{proof}

\subsection{Extension to with Non-integral Resources}
In this section, we consider a general case when both integral and non-integral cases are involved with the respective sparsities of $\ell_1$ and $\ell_2$. Let $B=\min_{k \in \cK_2} B_k$ be the smallest budget over all \emph{non-integral resources} and $\ell=\ell_1+\ell_2$. We show that the performance of $\alg$ designed for integral resources only will deteriorate slightly when applied to the case when both integral and non-integral resources are involved, provided that $B$ is moderately large. The analysis in this section is parallel to that in Section~\ref{sec:int}. As before, for each $e \in E^1, e'\in E^2$ and $t\in[T]$, let $X_e$ be a random number of copies of $e$ selected in Phase I and $Y_{e',t}$ indicate if $e'$ is selected at $t$ in Phase II. We focus on analyzing the algorithm $\alg(\eta,\alp)$ with $\eta=\alp=1/(2\ell)$. 

For Phases I and II, we have the following two respective lemmas. 
\begin{lemma}\label{lem:ratio-3}
For each $e\in E^1$, we have $\E[X_e] \ge  x_e^* \frac{1-\ep_1}{4\ell}$ with $\ep_1 \doteq 2/(B-2)$. 
\end{lemma}

\begin{lemma}\label{lem:ratio-4}
For each $e\in E^2$ and $t \in [T]$, we have $\E[Y_{e,t}] \ge y_{e,t}^* \frac{1-\ep_2}{4 \ell}$ with $\ep_2=2\ell_2 \exp(-B/12+1/6)$.  
\end{lemma}

We defer the proofs of Lemmas~\ref{lem:ratio-3} and~\ref{lem:ratio-4} to the appendix.

\xhdr{Proof of main Theorem~\ref{thm:main-2}}. 
\begin{proof}
From Lemmas~\ref{lem:ratio-3} and \ref{lem:ratio-4}, the final online competitive ratio achieved by $\alg$ with $\eta=\alp=1/(2\ell)$ should be  at least $\frac{1}{4 \ell} (1-\ep)$, where $\ell=\ell_1+\ell_2$ and $\ep=2\max \big(1/(B-2), \ell_2 \exp(-B/12+1/6) \big)$ with $B$ being the minimum budget over all non-integral resources.  
\end{proof}

\section{Experiments}

% In this section, we describe our experimental results on the real dataset and defer the results on synthetic datasets to the full version.

\begin{figure*}[!ht]
	\centering
	\includegraphics[width=\textwidth]{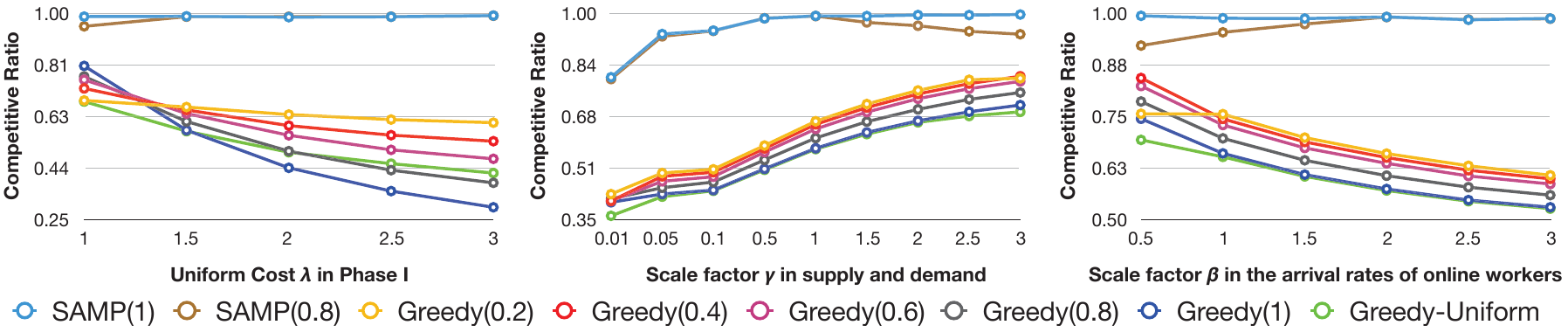}
    \caption{Real dataset: the change of competitive ratios when varying uniform cost in Phase I (Left), scale factor in supply and demand (Middle), and scale factor in the arrival rates of online workers (Right). Default setting is $\lam=1.5, \gam=2, \beta=1$.}
    \vspace{-3mm}
    \label{fig:cr_real} %% label for entire figure 
\end{figure*}
\vspace{-2mm}

\xhdr{The bike sharing dataset and preprocessing}.
We use the Citibike dataset in New York City,\footnote{\url{https://www.citibikenyc.com/system-data}}
which contains the trip histories for trips and real-time data for hundreds of stations across Manhattan, Brooklyn, Queens and Jersey City.
Each trip record includes the starting station, destination station, the starting time at which the trip was initiated and the ending time at which the trip ended.
In our experiments, we focus on the city of Manhattan with $374$ active bike stations.

We partition all stations into $50$ sites by applying $K$-medians clustering method, based on the geographical location and Manhattan distance. Note that the clustering method may converge to a local optimum.
% \bluee{We partition all stations into $50$ sites by applying $K$-medoids clustering method,
% based on the geographical location and Manhattan distance.}
%based on the bike usage patterns among all stations. 
% We partition all stations into $50$ sites by applying $K$-medians clustering method, based on the bike usage patterns among all stations. 
Basically, we have two kinds of sites, \emph{supplying sites} with a large check-in amount and thus abundant bikes available, and \emph{demanding sites} with a high check-out demand. 
%For simplicity, 
Since not all sites are highly imbalanced,
we focus on the rebalancing among the top $10$ \emph{supplying sites} and the top $10$ \emph{demanding sites}, which are denoted by $\cC_1$ and $\cC_2$, respectively.
The total set of task types $I = \cC_1 \times \cC_2$ and each $i =(a,b)\in I$ represents a task type of moving one bike from $a \in \cC_1$ to $b \in \cC_2$. Focus on the morning rush hour of 8 AM to 9 AM from Monday to Thursday in September and October of 2017 (34 days in total). 
For each $a \in \cC_1$ and $b \in \cC_2$, we first compute the average supply ($s_a$) at site $a$ and the average demand ($d_b$) at $b$ and then set the final supply  $c_a = \gamma \cdot s_a$ for $a$ and the final demand $c_b = \gamma \cdot d_b$ for $b$, where $\gamma$ is a parameter adjusting the scale of supply and demand overall.
We assume that for TBR, one unit task of each type incurs a uniform cost $\lam \ge 1$\footnote{Later we will set the maximum cost of assignments in CBR is $1$; thus, $\lam$ can be viewed as the relative cost of TBR to CBR.}, which captures 
the labor cost of loading  one bike up and off a truck and the amortized cost of renting and operating a truck. 

Now we discuss the setting for CBR. Let $\cL$ be the collection of $50$ sites. For each pair $(c,d) \in \cL \times \cL$, we create a worker type $j=(c,d)$ such that $j$ has the starting and ending sites of $c$ and $d$, respectively. For a given worker of type $j=(c,d)$, we set that it is qualified for task $i=(a,b)$ if $D(c,a) + D(b,d) \leq \tau$, 
where $D(\cdot)$ refers to the Manhattan distance between two sites and $\tau$ is
a given threshold for the additional walking distance. For each assignment $e=(i,j)$, the cost of payment for worker $j$ is set as $\cost(e)=\rho + (1-\rho)\frac{D(c,a) + D(b,d)}{\tau}$, which consists of a basic rate $\rho \in [0,1]$ (a parameter) and an incentive part proportional to the travel distance due to the rebalancing work. Here we scale down such that the largest cost of payment is $1$ among all assignments in CBR. 

For each worker of type $j$,  we first compute the average number of trips $r'_j$ during the rush hour over the 34 days. Then we introduce a parameter $\beta$ to adjust the degree of peak hour and set the final arrival rate $r_j=\beta \cdot r'_j$ and the total arrivals $T=\sum_{j \in J} r_j$. We consider a special arrival setting of KIID, where arrival distributions are assumed identical and independent throughout the online phase. This is mainly due to the short time window considered here (\ie the rush hour of 8 AM to 9 AM). The same setting is adopted by~\cite{aaai-19-stable,AAMAS18} for ridesharing and crowdsourcing. After splitting the online phase into $T$ rounds, 
the arrival distributions are set as $p_{jt}=p_j= r_j / T$ for every $t \in [T]$ and $j \in J$ such that $\sum_{j \in J}p_j = 1$.
For each task $i$, we assign it with a uniformly random weight $w_i\in [0,1]$ and assume all assignments with respect to $i$ return a profit/weight of $w_i$ regardless of TBR or CBR. Set the default total budget as $\kappa \cdot A$, where $\kappa=0.5$ and $A$ is the expected minimum cost of payment to complete all tasks. Since the average cost of CBR is less than that of TBR for each task, we set $A$ as the product of the average cost of all assignments in CBR and the total number of all potential tasks (equal to $\min(\sum_{a \in \cC_1}c_a,\sum_{b \in \cC_2}c_b)$).

\xhdr{Algorithms}. 
By default, we set $\alp=1$ and test $\alg(\eta)$ (which is short for $\alg(\eta,1)$) with $\eta=1$ and $\eta=0.8$, respectively, against several heuristic baselines, namely \gre and Greedy-Uniform. Note that for $\alg(1)$ and $\alg (0.8)$, they independently sample $ \rd(x_e^*) $ and $ \rd( 0.8 x_e^*)$ copies of edges for each $e \in E^1$. For each heuristic $\gre(\del)$ with $\del \in \{0.2,0.4,0.6,0.8,1\}$, it first allocates at most a fraction $\del$ of the total budget ($\del \cdot \bfB$) to  Phase I and all the rest goes to Phase II. In Phase I, $\gre(\del)$ greedily chooses edges $e$ sequentially in a decreasing order of their weights until allocated budgets are exhausted. In Phase II when a worker of type $j$ arrives, $\gre (\del)$ will always assign  $j$ with an assignment $e \in E_j^2$ which has the maximum weight among all safe choices. Greedy-Uniform can be viewed as a randomized version of \gre: it first samples a uniform value $q \in [0,1]$ and then runs $\gre(q)$. For each given setting, we run all algorithms for $1000$ times and take the average as the final performance. We compare the performance of each algorithm against the optimal value to the benchmark LP~\eqref{LP:offline-a} and use that ratio as the final competitive ratio achieved.
Note that the algorithms designed for the MBOA model~\cite{dickerson2019online} are applicable in Phase II only, a future direction is to construct a new LP for these algorithms which can capture the features of the \ts model.

\xhdr{Results on the real dataset}.
Figure~\ref{fig:cr_real} (Left) shows the effect of $\lambda$, which captures the relative ratio of cost in Phase I to that in Phase II. 

The results show that for each given $\lam \ge 1.5$, the performances of $\gre(\del)$ decreases as $\del$ increases.
What is more, the neighboring gaps widen when $\lam$ increases. Observe that a larger value of $\lam$ implies more costly TBR is and thus, it will be more profitable to allocate more budgets to Phase II. This explains why $\gre(\del)$ achieves a higher ratio when $\del$ is smaller (thus a higher fraction of budget goes to Phase II).  However, $\alg(1)$ and $\alg(0.8)$ universally beat all heuristics with a constant gap around $0.3$ in the competitive ratios. Figure~\ref{fig:cr_real} (Middle) shows the effect of $\gamma$, which captures the scale of supply and demand. Note that when $\gamma$ increases, we increase 
our default budget value $A$ proportionally. Results show that all heuristics have an increasing performance though $\alg(1)$ and $\alg(0.8)$ strictly dominate all of them. This is due to the fact that the more resources we have (with a larger $\gam$), the less planning we need. Thus, the superiority of our LP-based policies is reduced when resources are more abundant. Figure~\ref{fig:cr_real} (Right) shows the effect of $\beta$, which captures the different scale of arrival rates of worker types. When $\beta$ increases, we have more arrivals of low-cost workers in Phase II. Thus, we expect that the optimal policy will allocate more budget to Phase II. However, each   $\gre(\del)$ sticks with a fixed fraction of budget on each phase. This can explain why each $\gre(\del)$ has a decreasing performance. Again, our LP-based policies dominate all the greedy-based heuristics. 

Overall, Figure~\ref{fig:cr_real} suggests that our LP-based policies can optimize the budget allocation in Phases I and II to well respond to different costs in the two phases. Additionally, they have robust performances, which universally dominate all greedy-based heuristics under various settings. 

Furthermore, the ratios achieved by $\alg(1)$ and $\alg(0.8)$ are far larger than the lower bounds shown in Theorem~\ref{thm:main-2}, however. 
This is due to real-world settings are often faraway from the theoretical worst-case scenario.
\vspace{-2mm}
\section{Conclusions and Future Directions}

In this paper, we proposed a unified model \ts, which incorporates both offline and online resource allocation into a single framework.
Extensive experimental results on the real dataset show the robustness and effectiveness of our LP-based approaches in a wide range of settings. 
We observe that competitive ratios on experimental results, although guaranteed, are much above those theoretical lower bounds. This fact suggests that the hypothetical worst-case scenario has a structure faraway from those in the real world. A natural direction is to figure out the reasons behind the big gap between the actual performances and theoretical lower bounds. Can we improve the current ratio of $1/(4\ell)$ further? Another interesting question is whether we can remove the requirement of a lower bound of budgets on all non-integral resources, as shown in Theorem~\ref{thm:main-2}.
\section*{Acknowledgments}
Yifan Xu and Jun Tao's research was partially supported by the CERNET Southeastern China (North) Regional Network Center, CIGRI and the National Key Research and Development Program of China (No. 2018YFB1800205).
Pan Xu was partially supported by NSF CRII Award IIS-1948157.
Jianping Pan's research was partially supported by NSERC and CFI of Canada.
The authors would like to thank the anonymous reviewers for their helpful feedback.

\clearpage
{\small
\bibliographystyle{named}
\bibliography{refs}

\begin{thebibliography}{}

\bibitem[\protect\citeauthoryear{Ashlagi \bgroup \em et al.\egroup
  }{2019}]{ashlagi2019edge}
Itai Ashlagi, Maximilien Burq, Chinmoy Dutta, Patrick Jaillet, Chris Sholley,
  and Amin Saberi.
\newblock Edge weighted online windowed matching.
\newblock In {\em Proceedings of the Nineteenth ACM Conference on Economics and
  Computation}, 2019.

\bibitem[\protect\citeauthoryear{Assadi \bgroup \em et al.\egroup
  }{2015}]{assadi2015online}
Sepehr Assadi, Justin Hsu, and Shahin Jabbari.
\newblock Online assignment of heterogeneous tasks in crowdsourcing markets.
\newblock In {\em AAAI-HComp}, 2015.

\bibitem[\protect\citeauthoryear{Bansal \bgroup \em et al.\egroup
  }{2012}]{bansal2012solving}
Nikhil Bansal, Nitish Korula, Viswanath Nagarajan, and Aravind Srinivasan.
\newblock Solving packing integer programs via randomized rounding with
  alterations.
\newblock {\em Theory of Computing}, 8(1), 2012.

\bibitem[\protect\citeauthoryear{Bei and Zhang}{2018}]{BeiZ18}
Xiaohui Bei and Shengyu Zhang.
\newblock Algorithms for trip-vehicle assignment in ride-sharing.
\newblock In {\em AAAI}, 2018.

\bibitem[\protect\citeauthoryear{Chen \bgroup \em et al.\egroup
  }{2016}]{chen2016dynamic}
Xi~Chen, Will Ma, David Simchi-Levi, and Linwei Xin.
\newblock Dynamic recommendation at checkout under inventory constraint.
\newblock {\em http://dx.doi.org/10.2139/ssrn.2853093}, 2016.

\bibitem[\protect\citeauthoryear{Chenthamarakshan \bgroup \em et al.\egroup
  }{2012}]{chenthamarakshan2012systems}
V.E. Chenthamarakshan, N.~Kambhatla, R.C. Kanjiranthinkal, A.K.R. Singh, and
  K.~Visweswariah.
\newblock Systems and methods for matching candidates with positions based on
  historical assignment data, May~17 2012.
\newblock US Patent App. 12/944,868.

\bibitem[\protect\citeauthoryear{Dickerson \bgroup \em et al.\egroup
  }{2018a}]{DickersonAAAI18}
John~P. Dickerson, Karthik~Abinav Sankararaman, Aravind Srinivasan, and Pan Xu.
\newblock Allocation problems in ride-sharing platforms: Online matching with
  offline reusable resources.
\newblock In {\em AAAI}, 2018.

\bibitem[\protect\citeauthoryear{Dickerson \bgroup \em et al.\egroup
  }{2018b}]{AAMAS18}
John~P. Dickerson, Karthik~Abinav Sankararaman, Aravind Srinivasan, and Pan Xu.
\newblock Assigning tasks to workers based on historical data: Online task
  assignment with two-sided arrivals.
\newblock In {\em AAMAS}, pages 318--326, 2018.

\bibitem[\protect\citeauthoryear{Dickerson \bgroup \em et al.\egroup
  }{2019}]{dickerson2019online}
John~P Dickerson, Karthik~Abinav Sankararaman, Kanthi~Kiran Sarpatwar, Aravind
  Srinivasan, Kun-Lung Wu, and Pan Xu.
\newblock Online resource allocation with matching constraints.
\newblock In {\em AAMAS}, pages 1681--1689, 2019.

\bibitem[\protect\citeauthoryear{Duan and Wu}{2019}]{duan2019optimizing}
Yubin Duan and Jie Wu.
\newblock Optimizing the crowdsourcing-based bike station rebalancing scheme.
\newblock In {\em ICDCS}, 2019.

\bibitem[\protect\citeauthoryear{Fortuin \bgroup \em et al.\egroup
  }{1971}]{fortuin1971}
C.~M. Fortuin, P.~W. Kasteleyn, and J.~Ginibre.
\newblock Correlation inequalities on some partially ordered sets.
\newblock {\em Communications in Mathematical Physics}, 22(2):89--103, 1971.

\bibitem[\protect\citeauthoryear{F{\"u}redi \bgroup \em et al.\egroup
  }{1993}]{furedi1993}
Zolt{\'a}n F{\"u}redi, Jeff Kahn, and Paul~D. Seymour.
\newblock On the fractional matching polytope of a hypergraph.
\newblock {\em Combinatorica}, 13(2), 1993.

\bibitem[\protect\citeauthoryear{Ghodsi \bgroup \em et al.\egroup
  }{2012}]{ghodsi2012multi}
Ali Ghodsi, Vyas Sekar, Matei Zaharia, and Ion Stoica.
\newblock Multi-resource fair queueing for packet processing.
\newblock {\em ACM SIGCOMM}, 2012.

\bibitem[\protect\citeauthoryear{Ghodsi \bgroup \em et al.\egroup
  }{2013}]{ghodsi2013choosy}
Ali Ghodsi, Matei Zaharia, Scott Shenker, and Ion Stoica.
\newblock Choosy: Max-min fair sharing for datacenter jobs with constraints.
\newblock In {\em ACM ECCS}, 2013.

\bibitem[\protect\citeauthoryear{Guedes \bgroup \em et al.\egroup
  }{2014}]{recentPolice}
R.~Guedes, V.~Furtado, and T.~Pequeno.
\newblock Multiagent models for police resource allocation and dispatch.
\newblock In {\em 2014 IEEE Joint Intelligence and Security Informatics
  Conference}, 2014.

\bibitem[\protect\citeauthoryear{Haider \bgroup \em et al.\egroup
  }{2018}]{haider2018inventory}
Zulqarnain Haider, Alexander Nikolaev, Jee~Eun Kang, and Changhyun Kwon.
\newblock Inventory rebalancing through pricing in public bike sharing systems.
\newblock {\em European Journal of Operational Research}, 270(1):103--117,
  2018.

\bibitem[\protect\citeauthoryear{Ho and Vaughan}{2012}]{ho2012online}
Chien-Ju Ho and Jennifer~Wortman Vaughan.
\newblock Online task assignment in crowdsourcing markets.
\newblock In {\em AAAI}, 2012.

\bibitem[\protect\citeauthoryear{Huang \bgroup \em et al.\egroup
  }{2019}]{huang2019optimal}
Taoan Huang, Bohui Fang, Hoon Oh, Xiaohui Bei, and Fei Fang.
\newblock Optimal trip-vehicle dispatch with multi-type requests.
\newblock In {\em AAMAS}, pages 2024--2026, 2019.

\bibitem[\protect\citeauthoryear{Joe-Wong \bgroup \em et al.\egroup
  }{2013}]{joe2013multiresource}
Carlee Joe-Wong, Soumya Sen, Tian Lan, and Mung Chiang.
\newblock Multiresource allocation: Fairness-efficiency tradeoffs in a unifying
  framework.
\newblock {\em IEEE/ACM TON}, 2013.

\bibitem[\protect\citeauthoryear{Lee \bgroup \em et al.\egroup
  }{1979}]{Lee1979}
Sang~M. Lee, Lori~Sharp Franz, and A.~James Wynne.
\newblock Optimizing state patrol manpower allocation.
\newblock {\em Journal of the Operational Research Society}, 30(10), Oct 1979.

\bibitem[\protect\citeauthoryear{Li \bgroup \em et al.\egroup
  }{2018}]{li2018dynamic}
Yexin Li, Yu~Zheng, and Qiang Yang.
\newblock Dynamic bike reposition: A spatio-temporal reinforcement learning
  approach.
\newblock In {\em ACM SIGKDD}, 2018.

\bibitem[\protect\citeauthoryear{Li \bgroup \em et al.\egroup
  }{2020}]{Li2020TripVehicleAA}
Songhua Li, M.~Li, and Victor C.~S. Lee.
\newblock Trip-vehicle assignment algorithms for ride-sharing.
\newblock In {\em COCOA}, 2020.

\bibitem[\protect\citeauthoryear{Liu \bgroup \em et al.\egroup
  }{2016}]{liu2016rebalancing}
Junming Liu, Leilei Sun, Weiwei Chen, and Hui Xiong.
\newblock Rebalancing bike sharing systems: A multi-source data smart
  optimization.
\newblock In {\em ACM SIGKDD}, 2016.

\bibitem[\protect\citeauthoryear{Lowalekar \bgroup \em et al.\egroup
  }{2018}]{Patrick-18-JAI}
Meghna Lowalekar, Pradeep Varakantham, and Patrick Jaillet.
\newblock Online spatio-temporal matching in stochastic and dynamic domains.
\newblock {\em Artificial Intelligence}, 261:71 -- 112, 2018.

\bibitem[\protect\citeauthoryear{Ma and Simchi-Levi}{2017}]{ma2017online}
Will Ma and David Simchi-Levi.
\newblock Online resource allocation under arbitrary arrivals: Optimal
  algorithms and tight competitive ratios.
\newblock {\em http://dx.doi.org/10.2139/ssrn.2989332}, 2017.

\bibitem[\protect\citeauthoryear{O'Mahony and Shmoys}{2015}]{o2015data}
Eoin O'Mahony and David~B Shmoys.
\newblock Data analysis and optimization for ({C}iti) bike sharing.
\newblock In {\em AAAI}, 2015.

\bibitem[\protect\citeauthoryear{Pan \bgroup \em et al.\egroup
  }{2019}]{pan2019drl}
Ling Pan, Qingpeng Cai, Zhixuan Fang, Pingzhong Tang, and Longbo Huang.
\newblock A deep reinforcement learning framework for rebalancing dockless bike
  sharing systems.
\newblock In {\em AAAI}, 2019.

\bibitem[\protect\citeauthoryear{Raviv \bgroup \em et al.\egroup
  }{2013}]{raviv2013static}
Tal Raviv, Michal Tzur, and Iris~A Forma.
\newblock Static repositioning in a bike-sharing system: models and solution
  approaches.
\newblock {\em EURO Journal on Transportation and Logistics}, 2(3):187--229,
  2013.

\bibitem[\protect\citeauthoryear{Shumate and
  Crowther}{1966}]{shumate1966quantitative}
Robert~P Shumate and Richard~F Crowther.
\newblock Quantitative methods for optimizing the allocation of police
  resources.
\newblock {\em J. Crim. L. Criminology \& Police Sci.}, 57, 1966.

\bibitem[\protect\citeauthoryear{Singla \bgroup \em et al.\egroup
  }{2015}]{singla2015incentivizing}
Adish Singla, Marco Santoni, G{\'a}bor Bart{\'o}k, Pratik Mukerji, Moritz
  Meenen, and Andreas Krause.
\newblock Incentivizing users for balancing bike sharing systems.
\newblock In {\em AAAI}, 2015.

\bibitem[\protect\citeauthoryear{Wang \bgroup \em et al.\egroup
  }{2018}]{wang2018on-adv}
Xinshang Wang, Van-Anh Truong, and David Bank.
\newblock Online advance admission scheduling for services with customer
  preferences.
\newblock {\em arXiv preprint arXiv:1805.10412}, 2018.

\bibitem[\protect\citeauthoryear{Yi \bgroup \em et al.\egroup
  }{2007}]{yi2007matching}
Xing Yi, James Allan, and W~Bruce Croft.
\newblock Matching resumes and jobs based on relevance models.
\newblock In {\em SIGIR}, 2007.

\bibitem[\protect\citeauthoryear{Zhao \bgroup \em et al.\egroup
  }{2019}]{aaai-19-stable}
Boming Zhao, Pan Xu, Yexuan Shi, Yongxin Tong, Zimu Zhou, and Yuxiang Zeng.
\newblock Preference-aware task assignment in on-demand taxi dispatching: An
  online stable matching approach.
\newblock In {\em AAAI}, 2019.

\end{thebibliography}
}

\onecolumn
\section{Appendix}

\subsection{Proof of Lemma~\ref{lem:ratio-3}}

\begin{proof}
Consider a given $e \in E^1$. For each $e'\in E^1$  including $e'=e$, let $Z_{e'}=\rd(\eta x_{e'}^*)$ which takes two possible values: $z_{e'}^{-}=\lfloor\eta x_{e'}^*  \rfloor$ and $z_{e'}^{+}=\lceil\eta x_{e'}^*  \rceil$. Let $\SF_{e,k}$ be the event that $e$ is safe with respect to the resource $k$. From Lemma~\ref{lem:ratio-1}, we have for any integral resource $k \in \cS_e \cap \cK_1$,
$$ \Pr[ \SF_{e,k}| Z_e=z] \ge (1-\eta), \forall z \in \{z_e^{-}, z_e^{+}\}$$ 

Now we analyze the event of $\SF_{e,k}$ for a non-integral resource $k$. Focus on a given non-integral resource $k \in \cS_e \cap \cK_2$ with $a_{e,k}>0$.  Let $\kappa_1=\sum_{e'\neq e} x_{e'}^* a_{e',k}$ and $\kappa_2=x_e^* a_{e,k}$. Note that $\kap_1+\kap_2 \le B_k$. 
\begin{align*}
& \Pr[ \neg \SF_{e,k}|Z_e=z_e^{+}] =  \Pr\Big[ \sum_{\footnotesize e'\neq e} Z_{e'} a_{e',k}+\lceil \eta x_e^* \rceil a_{e,k} >B_k\Big] \\
&\le  \Pr\Big[ \sum_{ e'\neq e} Z_{e'} a_{e',k}>B_k-(1+\eta x_e^* ) a_{e,k}\Big] \\
& \le  \frac{\E[\sum_{ e'\neq e} Z_{e'} a_{e',k}]}{B_k-1-\eta x_e^* a_{e,k} } =\frac{\eta \kappa_1}{B_k-1-\eta \kappa_2}
\end{align*}

Consider these two cases. Case (1): $\kap_2 \ge 2$. 
Notice that $\eta=1/(2\ell) \le 1/2$. Thus, $B_k-1-\eta \kappa_2=B_k-\kap_2+(1-\eta) \kap_2-1 \ge \kap_1$, which is followed by  $\Pr[ \neg \SF_{e,k}|Z_e=z_e^{+}] \le \eta$. Case (2): $\kap_2<2$. In this case, 
\begin{align*}
&  \Pr[ \neg \SF_{e,k}|Z_e=z_e^{+}]  \le \frac{\eta \kappa_1}{B_k-1-\eta \kappa_2} \le   \frac{\eta \kappa_1}{B_k-2} \le  \frac{\eta B_k}{B_k-2}
\\
& =  \eta \Big(1+\frac{2}{B_k-2} \Big) \le \eta \Big(1+\frac{2}{B-2} \Big)  \doteq \eta (1+\ep_1)
\end{align*}
We can verify that $\Pr[ \neg \SF_{e,k}|Z_e=z_e^{-}] \le  \Pr[ \neg \SF_{e,t}|Z_e=z_e^{+}]$ for any $k$. Thus, we claim that $\Pr[\SF_{e,k}|Z_e=z] \ge (1- \eta(1+\ep_1))$, for any choice of $z$ and any $k \in \cS_e$. By applying the FKG inequality, we have for any $z \in \{z_e^{-},z_e^{+}\}$,
$$\Pr[\SF_e |Z_e=z] \doteq \Pr[ \wedge_{k \in \cS_e}\SF_{e,k}|Z_e=z] \ge (1- \eta(1+\ep_1))^\ell$$
Thus, $\E[X_e] \ge \E[Z_e ]  \cdot (1-\eta (1+\ep_1))^\ell =x_e^* \cdot \eta \cdot (1-\eta (1+\ep_1))^\ell \ge x_e^* (1-\ep_1)/(4\ell)$. (Note that $\eta=1/(2\ell)$.) 
\end{proof}

\subsection{Proof of Lemma~\ref{lem:ratio-4}}\label{sec:app}

%\LemFour*
\begin{proof}

Consider a given $t \in [T]$, $e=(i,j)\in E^2$ and $k \in \cS_e$. Let $\tau_1=\sum_{e' \in E^1} x_{e'}^* a_{e',k}$ and $\tau_2=\sum_{e' \in E^2, t'<t} y_{e',t'}^* a_{e',k}$. From  Constraint~\eqref{cons:res_1} in \LP~\eqref{LP:offline-a}, $\tau_1+\tau_2 \le B_k$. 

Assume $j$ arrives at time $t$, which occurs with probability $p_{jt}$. Let $U_1$ and $U_{2,t}$ be the respective (random) consumptions of resource $k$ during Phase I and at the beginning of time $t$ of  Phase II. Notice that the event occurs that $e$ is safe at $t$ with respect to resource $k$, denoted by $\SF_{e,t,k}$, iff $U_1+U_{2,t} \le B_k-1$ (since $a_{e,k}=1$). For each given integral resource $k \in \cS_e \cap \cK_1$, $\Pr[\neg \SF_{e,t,k}] \le \alp$ from Lemma~\ref{lem:ratio-2} (note that $\alp=\eta=1/(2\ell)$). Now focus on a given non-integral resource $k \in \cS_e \cap \cK_2$. Observe that 
$$\Pr[\neg \SF_{e,t,k}]  \le \Pr[U_1+U_{2,t} > B_k-1]$$

%From the analysis of Lemma~\ref{lem:ratio-2}, we see that $\E[U_1+U_{2,t}] \le \eta \tau_1+\alp \tau_2 \le \alp B_k$. 

For each $e \in E^1$, let $Z_e=\rd(\eta x_e^*)$. Assume $\eta x_e^*=z_e^{-}+\kap_e$, where $z_e^{-}=\lfloor \eta x_e^* \rfloor$ and WLOG assume $0<\kap_e<1$. Then $Z_e $ can be viewed as a sum of $z_e^{-}$ independent Bernoulli random variables each of which has mean $1$ plus one additional Bernoulli random variable with mean $\kap_e$. Note that $U_1 \le \sum_{e' \in E^1} a_{e',k} Z_{e'}$ with probability $1$ according to $\alg$. 

 For each $t'<t$ and $j \in J$, let $\bo_{j,t'}$ indicate if $j$ arrives at $t'$. For each $e' \in E^2_j$, let $\bo_{e',t'}$ indicate if $e'$ is sampled when $j$ comes at $t'$. From the analysis of Lemma~\ref{lem:ratio-2}, $U_{2,t} \le \sum_{t'<t}\sum_{j \in J} \sum_{e' \in E_j^2} \bo_{j,t'} \cdot \bo_{e',t'} \cdot a_{e',k}$. Therefore,
 
 \begin{align}
\Pr[\neg \SF_{e,t,k}]  & \le \Pr[U_1+U_{2,t} > B_k-1]  \le \Pr\Big[\sum_{e' \in E^1} a_{e',k} Z_{e'} +\sum_{t'<t}\sum_{j \in J} \sum_{e' \in E_j^2} \bo_{j,t'} \cdot \bo_{e',t'} \cdot a_{e',k} >B_k-1 \Big] \\
& \doteq \Pr[H_{e,t} >B_k-1]  \le \exp \left( \frac{-\alp B_k \Big(\frac{B_k-1}{\alp B_k}-1 \Big)^2}{\frac{B_k-1}{\alp B_k}+1 } \right) \le \exp\Big(-\frac{B}{12}+\frac{1}{6} \Big) \label{ineq:lem-4}
 \end{align}
 The first inequality in~\eqref{ineq:lem-4} is due to the Chernoff bound where $H_{e,t}$ can be viewed a sum of independent or negatively correlated Bernoulli random variables and 
   $\E[H_{e,t}] \le \alp \tau_1+\eta \tau_2 \le \alp B_k$ from the analysis in Lemma~\ref{lem:ratio-2}. The second inequality is due to the facts that $B_k \ge B$ for all $k \in \cK_2$ and $\alp=1/(2 \ell) \le 1/2$. 

Summarizing our analysis over both integral and non-integral resources, we see
$$\Pr[\SF_{e,t}] =\Pr[\wedge_{k \in \cS_e} \SF_{e,k,t}]\ge 1-\ell_1 \alp-\ell_2  \exp\Big(-\frac{B}{12}+\frac{1}{6} \Big) \ge \frac{1}{2} (1-\ep_2)$$
where $\ep_2 \doteq 2\ell_2  \exp\Big(-\frac{B}{12}+\frac{1}{6} \Big)$. Therefore,
\begin{align*}
\E[Y_{e,t}]&=\Pr[(\bo_{j,t}=1) \wedge (\bo_{e,t}=1) \wedge (\wedge_{k \in \cS_e}\SF_{e,t,k})]  \ge p_{jt} \cdot (\alp y_{e,t}^*/p_{jt})\cdot  \frac{1}{2} (1-\ep_2) =y_{e,t}^* \frac{1-\ep_2}{4 \ell}
 \end{align*}
Thus, we get our claim.
\end{proof}

\end{document}